\documentclass[5pt]{article}

\usepackage[letterpaper]{geometry}
\usepackage{spconf,amsmath,amssymb,amsthm,epsfig,epstopdf,mathrsfs,bm,appendix,graphicx}
\usepackage[ruled]{algorithm2e}
\usepackage{enumitem}

\newtheorem{myDef}{Definition}

\usepackage{fancyhdr}
\begin{document}\sloppy

\def\x{{\mathbf x}}
\def\L{{\cal L}}

\title{Visual Confusion Label Tree For Image Classification}
%
\name{Yuntao Liu, Yong Dou, Ruochun Jin, Rongchun Li}
\address{
National University of Defense Technology\\
National Laboratory for Parallel and Distributed Processing\\
Changsha, China, 410073 \\
liuyuntao.me@gmail.com, \{yongdou, jinruochun, rongchunli\}@nudt.edu.cn
}
%
%
%

\maketitle

\begin{abstract}
Convolution neural network models are widely used in image classification tasks. However, the running time of such models is so long that it is not the conforming to the strict real-time requirement of mobile devices. In order to optimize models and meet the requirement mentioned above, we propose a method that replaces the fully-connected layers of convolution neural network models with a tree classifier. Specifically, we construct a Visual Confusion Label Tree based on the output of the convolution neural network models, and use a multi-kernel SVM plus classifier with hierarchical constraints to train the tree classifier. Focusing on those confusion subsets instead of the entire set of categories makes the tree classifier more discriminative and the replacement of the fully-connected layers reduces the original running time. Experiments show that our tree classifier obtains a significant improvement over the state-of-the-art tree classifier by $4.3\%$ and $2.4\%$ in terms of top-1 accuracy on \emph{CIFAR-100} and \emph{ImageNet} datasets respectively. Additionally, our method achieves $124\times$ and $115\times$ speedup ratio compared with fully-connected layers on AlexNet and VGG16 without accuracy decline.
\end{abstract}
\section{Introduction}
\label{sec:intro}

Deep convolution neural network(CNN) models are developing rapidly and it has evolved to the state-of-the-art technique~\cite{He2015Deep} in image classification tasks. However, when applied to real-time applications on embedded devices where the power and storage are limited, CNN models can not meet the real-time demands because of its large amount of computation. Therefore, optimizing and accelerating these CNN models on embedded devices has become a challenge.

In view of this problem, researches have proposed a variety of compression and acceleration methods such as reducing the precision of multiplication and addition operations~\cite{Courbariaux2014Training}, setting the weights and inputs to binary codes~\cite{Rastegari2016XNOR}, a skillful integration among several effective methods~\cite{Han2016Deep} and structure changing ~\cite{Lin2013Network,Springenberg2014Striving,iandola2016squeezenet}. However, the computation of a CNN model mainly induced by the computation of fully-connected(FC) layers~\cite{Han2016Deep}. The methods mentioned above mainly focus on the compression on the convolution layers and have not resolved the huge computation problem of FC layers.

Another classification method based on a tree classifier, which is an appropriate method for large-scale recognition problems with thousands of categories, has received extensive attention and substantial development. There are several methods to construct the structure of the tree classifier such as leveraging the semantic ontologies (taxonomies)~\cite{li2010building,zhao2011large,deng2009imagenet}, learning label trees~\cite{bengio2010label} and probabilistic label trees~\cite{liu2013probabilistic}, learning visual trees~\cite{fan2012quantitative,fan2015hierarchical} and enhanced visual tree~\cite{zheng2017hierarchical}. Compared with the FC layers in the CNN models, the tree classifier has the advantage of small amount of calculation and the computation complexity of the tree classifier is $\mathcal{O}(log N)$~\cite{bengio2010label}. However, there has been no work that replaces FC layers with the tree classifier because most previous work construct the structure of the tree classifier by clustering directly from the image dataset. Previous methods do not utilize the information of FC layers so the accuracy is limited, which restricts the application of tree classification in accelerating depth CNN models. Moreover, this limitation also results in the separation of research on tree classification and deep CNN models and both of them can not benefit from each other.

\cite{jinconfusion} discovered that deep CNN models have visual confusions that is similar to human beings and we believe this characteristic can be used as the metric to construct the Label Tree. Therefore, we propose to use the community detection algorithm to construct the Label Tree, called {\bf Visual Confusion Label Tree(VCLT)}. With this method, we can fully utilize the information of FC layers in the CNN models. Compared with previous Label Tree building methods, the advantage of the VCLT is that there is no need to manually set parameters and do clustering tasks during tree construction. In addition, our VCLT is constructed directly based on the features in deep CNN models so it has a more reasonable structure which is beneficial for improving the accuracy of the tree classifier. Moreover, to the best of our knowledge, VCLT is the first effort that connects the CNN model and the Label Tree directly so the tree structure fully inherits the information contained in the FC layers.

There are two main contributions in this paper as follows.
\begin{itemize}
\item \emph{Visual Confusion Label Tree:} Our construction method is based on the hierarchical community detection algorithm. Using this algorithm on the output of FC layers we can construct a tree classifier whose structure is more suitable for deep CNN models. With this method we can improve the accuracy of the tree classifier compared with previous work by $2.4\%$--$4.3\%$ and we can prove in theory.
\item \emph{Replace the FC layers with the tree classifier:} After constructing the label tree, we replace the FC layers from the deep CNN models with our VCLT and we propose an effective algorithm to train our tree classifier. With this replacement we reduce the amount of computation in FC layers by $37\times$--$124\times$ without sacrificing the accuracy of original CNN-based methods.
\end{itemize}

\section{Label Tree in a Nutshell}
\label{sec:LabeltreeNutshell}

The concept of the Label Tree was first proposed in \cite{bengio2010label} aiming at classification and a label tree is a tree $T = (N,E,F,L)$ with $n+1$ indexed nodes $N = \{0,\dots,n\}$ where $E = \{(p_{1},c_{1}),(p_{|E|},c_{|E|})\}$ is a set of edges that are ordered pairs of parent and child node indices, $F = \{f_{1},\ldots,f_{n}\}$ are label predictors and label sets $L = \{l_{0},\ldots,l_{n}\}$ are associated to each node. Except the root of the tree, all other nodes have one parent and arbitrary number of children. The label sets indicate the set of labels to which a point should belong if it arrives at the given node. Classifying an example begins at the root node and for each edge leading to a child $(s,c) \in E$ one computes the score of the label predictor $f_{c}(x)$ which predicts whether the example $x$ belongs to the set of labels $l_{c}$. One takes the most confident prediction, traverses to that child node, and then repeats the process. Classification is complete when one arrives at a node that identifies only a single label that is the predicted class. More details about Label Trees can be found in \cite{bengio2010label, fan2015hierarchical, deng2011fast}.

\section{Visual Confusion Label Tree and Training}
\label{sec:vcltTraining}

\subsection{Definition of the Visual Confusion Label Tree}
\label{sec:vcltDef}

\begin{myDef}\label{def:confusionTree}
A Visual Confusion Label Tree is a tree $T=(\mathscr{N},\mathscr{V},\mathscr{E},\mathscr{L})$ with $k$ hierarchical layers $\mathscr{N}=\{n_1,\dots,n_k\}$ where $n_{id}$ denotes the number of nodes in the $id$th layer, the node sets $\mathscr{V}=\{V_1,\dots,V_k\}$ where $V_{id}$ is a set of nodes in the $id$th layer and $V_{id} = \{v_1,\dots,v_{n_{id}}\}$, branch edges $\mathscr{E}=\{(p_1,c_1),\dots,(p_{|\mathscr{E}|},p_{|\mathscr{E}|})\}$ which are ordered pairs of parent and child node indices and labels  sets $\mathscr{L}=\{L_1,\dots,L_s\}$ where $L_{id}$ is a label set of nodes in the $id$th layer and $L_{id} = \{l_1,\dots,l_{n_{id}}\}$ where $l_s$ denotes the label set of the $s$th node in this layer.
\end{myDef}

We extend the notation in~\cite{bengio2010label} to Definition \ref{def:confusionTree}. The number $k$ in Definition \ref{def:confusionTree}, which is the number of the hierarchical layers of VCLT, is equal to the number of iterations of hierarchical community detection algorithm run on a related confusion graph(detailed in Section \ref{sec:EstablishmentOfVCLT}).

\subsection{Visual Confusion Label Tree Establishing}
\label{sec:EstablishmentOfVCLT}

\begin{algorithm}\label{alg:establishCVT}
\caption{Establish Visual Confusion Label Tree of a $N$-category classification}
\LinesNumbered
\KwIn{
A $N$-category classification model $M$; A dataset $D$;
Top concern number $\tau(\tau\leq N)$;
}
\KwOut{
The VCLT $T=(\mathscr{N},\mathscr{V},\mathscr{E},\mathscr{L})$;
}
$G\Leftarrow GenerateConfusionGraph(M,D,N,\tau)$\;
$\mathscr{P}, \mathscr{C}\Leftarrow HierarchicalCommunityDetect(G)$\;
$k\Leftarrow length(\mathscr{P})$\;
\For{i from 1 to k}{
	$\mathscr{N} = \mathscr{N}\bigcup length(\mathscr{P}[i])$\;
	$\mathscr{V} = \mathscr{V}\bigcup \mathscr{P}[i]$\;
	$\mathscr{L} = \mathscr{L}\bigcup \mathscr{C}[i]$\;
}
\For{i from 2 to k}{
	\For{j from 1 to $\mathscr{N}$[i]}{
		\ForEach{$v_s \in \mathscr{P}[i][j]$}{
			$\mathscr{E} = \mathscr{E}\bigcup e_{\mathscr{V}[i][j],v_s}$\;
		}	
	}
}
return $T$\;
\end{algorithm}

Given a dataset $D$ and its corresponding classification model $M$, Algorithm \ref{alg:establishCVT} establishes the VCLT $T$ defined in Definition \ref{def:confusionTree} using confusion graph generation and community detection algorithm. There are 3 main steps in Algorithm \ref{alg:establishCVT}. The first is using the confusion graph generation algorithm to build a confusion graph. The second is using the hierarchical community detection algorithm to reveal communities in the confusion graph. The last is establishing a VCLT with the results of the second step.

Specifically, for the function ``GenerateConfusionGraph'', we utilize the confusion graph establishing algorithm from~\cite{jinconfusion}. This algorithm firstly normalizing the top-$\tau$ classification scores of each test sample and then accumulating each normalized score to the weight of the edge that connects the labeled category and the predicted category. For the function ``HierarchicalCommunityDetect'', we use the algorithm from ~\cite{blondel1008fast}. This algorithm is an iterative algorithm and it will continue running until the modularity is converged. This function outputs the $\mathscr{P}$ and $\mathscr{C}$. The $\mathscr{P}$ is a set of arrays which refer to the community set at each iteration and the $\mathscr{P}$ has the same structure as the $\mathscr{V}$ in Definition \ref{def:confusionTree}. The $\mathscr{C}$ is almost the same as the $\mathscr{P}$ except for the member in the $\mathscr{C}$ refer to the label set of communities at each iteration. In particular, as for line$6$ in Algorithm \ref{alg:establishCVT}, we just add the mark of the communities to $\mathscr{V}$ instead of the vertexes in these communities.

\begin{figure}[t]
\begin{minipage}[t]{1.0\linewidth}
  \centering
  \centerline{\includegraphics[width=6cm]{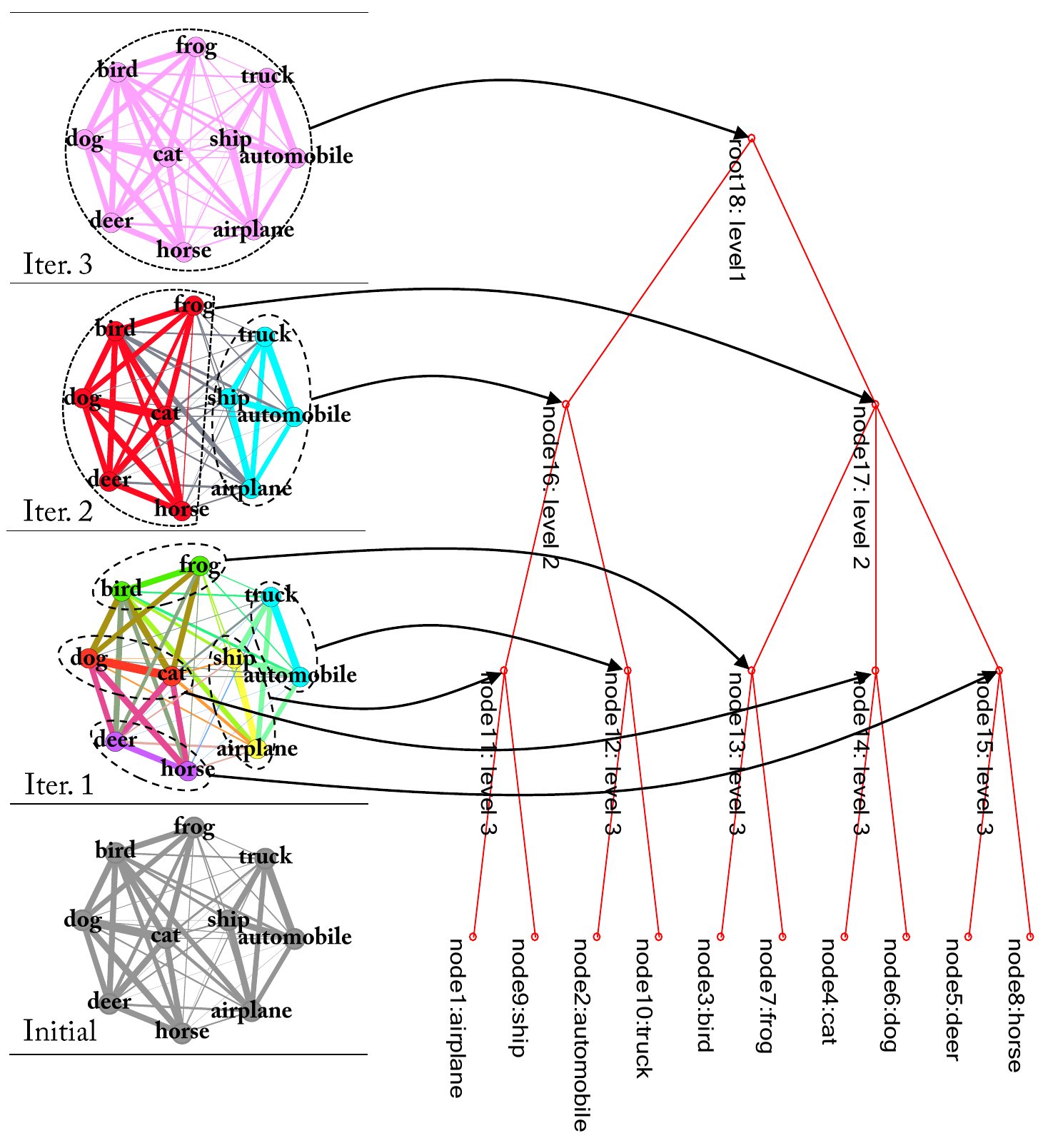}}
\end{minipage}
\caption{The construction process of the Visual Confusion Label Tree for CIFAR-10 image set.}
\label{fig:processOfEstablishCVTCIFAR10}
\end{figure}

We use the Algorithm \ref{alg:establishCVT} to construct a VCLT on \emph{CIFAR-10} dataset and the construction process is shown in Fig. \ref{fig:processOfEstablishCVTCIFAR10} where the confusion graph and the communities inside are on the left and the corresponding VCLT is on the right. The left side of Fig. \ref{fig:processOfEstablishCVTCIFAR10} is divided into four steps: Initial and Iteration $1$ to $3$. We apply the function ``GenerateConfusionGraph'' to generate a confusion graph which is shown at the Initial step. Each vertex represents one category in the dataset and the weight of each edge quantifies the confusion between two connected categories. For instance, the strongest link between ``dog'' and ``cat'' denotes that the model may highly probably confuse dogs with cats. Contrarily, the weak edge connecting ``dog'' and ``ship'' indicates that the confusion between them is weak. Then we use the function ``HierarchicalCommunityDetect'' on the confusion graph and get the community structure of the graph from fine-grained to course-grained at each iteration of this algorithm. At Iter. $1$, as is illustrated in Fig. \ref{fig:processOfEstablishCVTCIFAR10}, we get five fine-grained communities and set five corresponding nodes to the tree called ``node: level$3$''. As each member in a community refers to one category, we link the leaf nodes to level$3$ nodes. For instance, we link ``node$4$: cat'' and ``node$6$: dog'' to ``node$14$: level$3$''. At Iter. $2$, we get two coarse-grained communities based on communities detected in Iter. $1$ and each fine-grained community at Iter. $1$ is a member of the coarse-grained community at Iter. $2$. Then we link level$3$ nodes to level$2$ nodes based on this relationship. For instance, we link ``node$13$: level$3$'', ``node$14$: level$3$'' and ``node$15$: level$3$'' to ``node$17$: level$2$''. Similarly, at Iter. $3$, we link ``node$16$: level$2$'' and ``node$17$: level$2$'' to ``root$18$: level$1$'' and whole  finish the construction process.

As is proposed in~\cite{bengio2010label}, in order to achieve high classification accuracy, an ideal label tree should make the fine-grained categories contained in sibling leaf nodes under the same parent node as similar as possible while making the coarse-grained categories contained in parent nodes as dissimilar as possible. Our VCLT structure satisfies this because the categories in leaf nodes are strongly confused while those in parent nodes are weakly confused. Using Algorithm \ref{alg:establishCVT}, we also construct a VCLT on \emph{CIFAR-100} dataset shown in Fig. \ref{fig:CVTofCIFAR100}. Compared with the Enhanced Visual Tree (EVT) structure in~\cite{zheng2017hierarchical}, our VCLT is more reasonable. For example, EVT puts ``whale'', ``shark'', ``skyscraper'', ``rocket'' and ``mountain'' into one coarse-grained category while our VCLT divides them into three independent coarse-grained categories. Another example is that EVT divides ``bicycle'' and ``motorcycle'' into two different coarse-grained categories but our VCLT puts them into the same fine-grained category due to the strong visual similarity between them.


\begin{figure}[!t]
\begin{minipage}[t]{1.0\linewidth}
  \centering
  \centerline{\includegraphics[width=7cm]{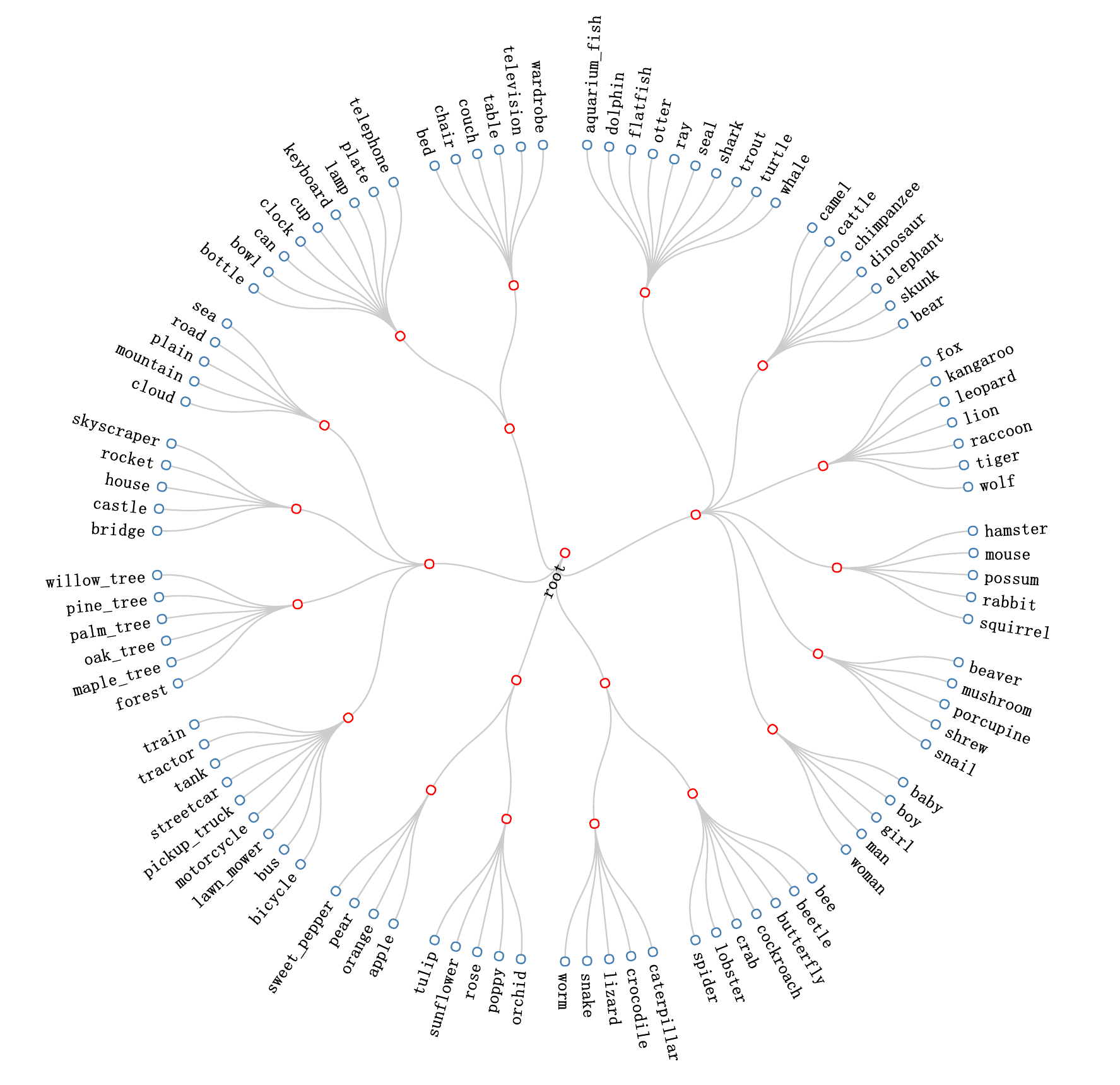}}
\end{minipage}
\caption{The Visual Confusion Label Tree for CIFAR-100 image set with 100 categories.}
\label{fig:CVTofCIFAR100}
\end{figure}

\subsection{Visual Confusion Label Tree Classifier Training}
\label{sec:vcltTrain}

Similar to \cite{fan2015hierarchical}, we develop a top-down approach to train classifiers on each node of the VCLT. One parent node contains a set of coarse-grained categories or a set of fine-grained categories. To make full use of these features, we apply a multi-kernel learning algorithm to train the classifier on each node. In order to control the inter-level error propagation, we add a constraint to our learning algorithm. The constraint aims to guarantee that an image must first be assigned to its parent node (higher-level non-leaf node) correctly if it can further be assigned to a child node (lower-level non-leaf node or leaf node). All these methods make the tree classifier over the VCLT more discriminative.

In order to discriminate a given coarse-grained or fine-grained categories on a node $c_j$ from its sibling nodes $\{s(c_j)\}$ under the same parent node $c_i=p(c_j)$, its multi-kernel SVM classifier is defined as:
\begin{equation}\label{equ:fClassifier}
f^l_{c_j}=K(\bm x,\bm x')+b
\end{equation}
where $l$ is the level of node $c_j$ and $K(\bm x,\bm x')$ is the multi-kernel which is defined as:
\begin{equation}\label{equ:kKernel}
K(\bm x,\bm x')=\sum^M_{m=1}{d_mK_m(\bm x,\bm x')}
\end{equation}
with:
\begin{equation}\label{equ:kConstraint}
d_m\geq 0, \sum^M_{m=1}{d_m}=1
\end{equation}

In our method, we use common kernels such as the linear kernel, the polynomial kernel, and the Gaussian kernel.

We train each classifier node by node from the root to leaf nodes and use the strategy of SVM Plus~\cite{Alexander2000Probabilities} to train the multi-label SVM classifiers. Specifically, there is a set of labeled training images for $R$ sibling nodes $\{s(c_j)\}$ under the same parent node $c_i=p(c_j)$, $R\in [2,B]$ (B is the number of sibling nodes) and there are training samples $\Omega = \{(x^l_j,y^l_j)|c_j\in c_i\}$ ($l$ is the level of the sibling nodes), training the multi-kernel SVM Plus classifiers for $R$ sibling nodes is achieved by optimizing a objective function:
\begin{equation}\label{equ:svmOptimizewithoutConstraint}
\min_{f_0,f_{c_j},b,\xi,d}
\frac{1}{2}{\|f_{0}\|^2}+\frac{\lambda}{2}{\sum^R_{j=1}{\sum^M_{m=1}{\frac{1}{d_m}\|f_{c_j}\|^2}}}
+C\sum^R_{j=1}{\xi_j}
\end{equation}
\textbf {subject to:}
\begin{equation}\label{equ:svmConstraintwithoutLevelError}
\forall^R_{j=1}: y_jf_{c_j}=y_j(\sum^M_{m=1}{K_{m,c_j}(\bm x_j,\bm x_j')+b})\geq1-\xi_j,\xi_j> 0
\end{equation}
where $\xi_j$ indicates the slack variable, $\lambda$ is the positive regularization parameters, $C$ is the penalty term.

One problem of label tree is that the error propagation may have negative influence on the classification result. If a classification error happens on the parent node, the prediction of this sample will be wrong because the labels of leaf nodes under this misclassified parent node are all incorrect. In order to resolve this problem, we add an inter-level constraint to the SVM plus classifiers. Our strategy is that samples should be classified correctly at the parent level(level:$l+1$) if we want it to be further classified at the current node level(level:$l$). Thus, we add a constraint to guarantee that the score of current node classifier must be larger than the score of its parent node classifier, which can be denoted by Eq.(\ref{equ:svmConstraintwithLevelError2}). Therefore, we extend Eq.(\ref{equ:svmOptimizewithoutConstraint},\ref{equ:svmConstraintwithoutLevelError}) to:
\begin{equation}\label{equ:svmOptimizewithConstraint}
\min_{f_0,f_{c_j},b,\xi,d}
\frac{1}{2}{\|f^l_{0}\|^2}+\frac{\lambda}{2}{\sum^R_{j=1}{\sum^M_{m=1}{\frac{1}{d_m}\|f^l_{c_j}\|^2}}}
+C\sum^R_{j=1}{\xi_j}
\end{equation}
\textbf {subject to:}
\begin{equation}\label{equ:svmConstraintwithLevelError1}
\forall^R_{j=1}: y_jf_{c_j}=y_j(\sum^M_{m=1}{K_{m,c_j}(\bm x_j,\bm x_j')+b})\geq1-\xi_j,\xi_j> 0
\end{equation}
\begin{equation}\label{equ:svmConstraintwithLevelError2}
\forall^R_{j=1}: {f^l_{c_j}(x_j)}\geq f^{l+1}_{c_i}(x_j),(x_j,y_j)\in c_j\in c_i
\end{equation}

\section{Experiment}
\label{sec:experiment}

\subsection{Datasets and Experimental Settings}
\label{sec:experimentDatasetandSettings}

We use \emph{CIFAR-100}~\cite{Krizhevsky2009Learning} and \emph{ILSVRC2012}~\cite{deng2009imagenet} to evaluate the performance of the proposed classification method. \emph{CIFAR-100} has $60000$ images of $100$ categories. Each category has $600$ images in which $500$ for training and $100$ for validation. Divided into a training set and a validation set, \emph{ILSVRC2012} has over $120$ million images of $1000$ categories and is commonly used to evaluate  image classification algorithms. We use the training set for training and validation set for testing. The Mean Accuracy (MA) $(\%)$~\cite{zheng2017hierarchical} is used to capture the performance of each method. A PC with Intel Core i7 and 64GB memory is utilized to run all experiments.

\subsection{Comparison of different tree classifiers}
\label{sec:experimentDifferentTreeClassifiers}

In this section, we compare the classification accuracy of our proposed VCLT classifier with those of other state-of-the-art tree classifiers. Trained and tested with \emph{CIFAR-100} and \emph{ImageNet} datasets, we compare the MA of each of the following tree classifiers: {\bf semantic ontology}~\cite{li2010building}, {\bf label tree}~\cite{bengio2010label}, {\bf visual tree}~\cite{fan2015hierarchical} and the {\bf enhanced visual tree}~\cite{zheng2017hierarchical}. In order to train and test each model, we employ the \emph{DeCAF}~\cite{Jia2014Caffe} features extracted from the FC6 layer of the AlexNet model (its first FC layer) and the classification accuracy of each tree classifier (quantified in MA) is demonstrated in Table \ref{tab:accuracyOfDifferentTreeClassifiers}.


\begin{table}[h]
	\centering
	\begin{tabular}{|l|c|c|}
		\hline
		Approaches&CIFAR-100&ImageNet\\
		\hline
		Semantic ontology&$48.95\%$&$44.47\%$\\
		Label tree&$52.04\%$&$49.64\%$\\
		Visual tree&$51.30\%$&$51.03\%$\\
		Enhanced visual tree&$54.33\%$&$58.76\%$\\
		Visual confusion label tree&$\mathbf{58.63\%}$&$\mathbf{61.18\%}$\\
		\hline			
	\end{tabular}
	\caption{Classification accuracy of different tree classifiers.}\label{tab:accuracyOfDifferentTreeClassifiers}
\end{table}

From Table \ref{tab:accuracyOfDifferentTreeClassifiers}, we find the performance of Semantic ontology is the worst because its tree structure is constructed based on semantic space and the image classification process based on feature space. For another four methods based on feature space, the performance of Label tree is worse because of using OvR classifier to construct its tree structure, which is limited to sample imbalance and the performance of classifier. As for Visual tree, it uses the average features extracted directly from the dataset. Enhanced visual tree adopts the spectral clustering method that better reflects the diversity of categories, so its performance is better than the Visual tree. Our VCLT constructs the tree structure based on the confusion of the CNN model, which makes sibling nodes as close as possible and the parent nodes as far as possible. So this tree structure is more proper and we obtain a significant improvement over the Enhanced visual tree by $4.3\%$ and $2.4\%$.

\subsection{Comparison between our tree classifier and CNN models}
\label{sec:experimentCNNModels}

In this section, we compare the classification accuracy and test time of our VCLT with those of the corresponding CNN model. We choose AlexNet and VGG-Verydeep-16(VGG16) for this comparison. In the AlexNet-based experiment, we firstly train an AlexNet using the \emph{CIFAR-100} dataset. Then we employ the \emph{DeCAF} features extracted from the FC6 layer of AlexNet to train the corresponding VCLT classifier. The classification accuracy and test time of both are shown in Table \ref{tab:accuracyComparedOrgCNNCIFAR}. For a CNN model, the ``test time'' in Table \ref{tab:accuracyComparedOrgCNNCIFAR} is the average running time of its FC layers when processing one image. For the VCLT classifier, the ``test time'' is the average running time of the whole tree classifier when processing one image. As for the VGG16-based experiment, we do the same thing except using the features extracted from the FC14 layer of VGG16 to train the VCLT classifier. Table \ref{tab:accuracyComparedOrgCNNIMAGENET} illustrates this comparison using the \emph{ImageNet} dataset.



\begin{table}[h]
    \centering
    \begin{tabular}{|l|c|c|c|}
        \hline
        Approaches&Accuracy&Test time (ms)&Speedup\\
        \hline
        AlexNet         &   $54.02\%$   &   $3.4215$     &  -\\
        VCLT\_AlexNet&   $\mathbf{58.63\%}$   &    $\mathbf{0.0275}$    &$\mathbf{124\times}$\\
        \hline
        VGG16&$72.21\%$&$4.3713$&-\\
        VCLT\_VGG16&$\mathbf{72.23\%}$&$\mathbf{0.0381}$&$\mathbf{115\times}$\\
        \hline
    \end{tabular}
    \caption{Experiment results on \emph{CIFAR-100}}
    \label{tab:accuracyComparedOrgCNNCIFAR}
\end{table}

\begin{table}[h]
    \centering
    \begin{tabular}{|l|c|c|c|}
        \hline
        Approaches&Accuracy&Test time (ms)&Speedup\\
        \hline
        AlexNet&$57.24\%$&$5.5263$&-\\
        VCLT\_AlexNet&$\mathbf{61.18\%}$&$\mathbf{0.1493}$&$\mathbf{37\times}$\\
        \hline
        VGG16&$\mathbf{71.53\%}$&$6.7981$&-\\
        VCLT\_VGG16&$66.71\%$&$\mathbf{0.2427}$&$\mathbf{28\times}$\\
        \hline
    \end{tabular}
    \caption{Experiment results on \emph{ImageNet}}
    \label{tab:accuracyComparedOrgCNNIMAGENET}
\end{table}

From Table \ref{tab:accuracyComparedOrgCNNCIFAR} and Table \ref{tab:accuracyComparedOrgCNNIMAGENET}, we can see that the classification accuracy of VCLT with AlexNet is around 3\% higher than that of the original AlexNet on both \emph{CIFAR-100} and \emph{ImageNet} datasets. In addition, on both datasets, the speedup ratios achieved by replacing the FC layers of AlexNet with our tree classifier are significant ($124\times$ on \emph{CIFAR-100} and $37\times$ on \emph{ImageNet}). Though the accuracy improvement of VCLT with VGG16 is trivial on \emph{CIFAR-100} and even negative on \emph{ImageNet}, the speedup ratios are remarkable, achieving $115\times$ and $28\times$ respectively, which demonstrates VCLT's promising potential to accelerate CNN-based applications.

\begin{table}[h]
    \centering
    \begin{tabular}{|l|c|c|}
        \hline
        Approaches&Accuracy\\
        \hline
        DeepCom&$57.24\%$\\
        VCLT\_DeepCom&$\mathbf{60.73\%}$\\
        \hline
        BWN&$57.08\%$\\
        VCLT\_BWN&$\mathbf{57.11\%}$\\
        \hline
        XNOR&$42.37\%$\\
        VCLT\_XNOR&$\mathbf{42.48\%}$\\
        \hline
    \end{tabular}
    \caption{Results on \emph{ImageNet}(compressed CNN models)}
    \label{tab:accuracyComparedCompressedCNN}
\end{table}



\begin{figure}[t]
\begin{minipage}[t]{1.0\linewidth}
  \centering
  \centerline{\includegraphics[width=7.5cm]{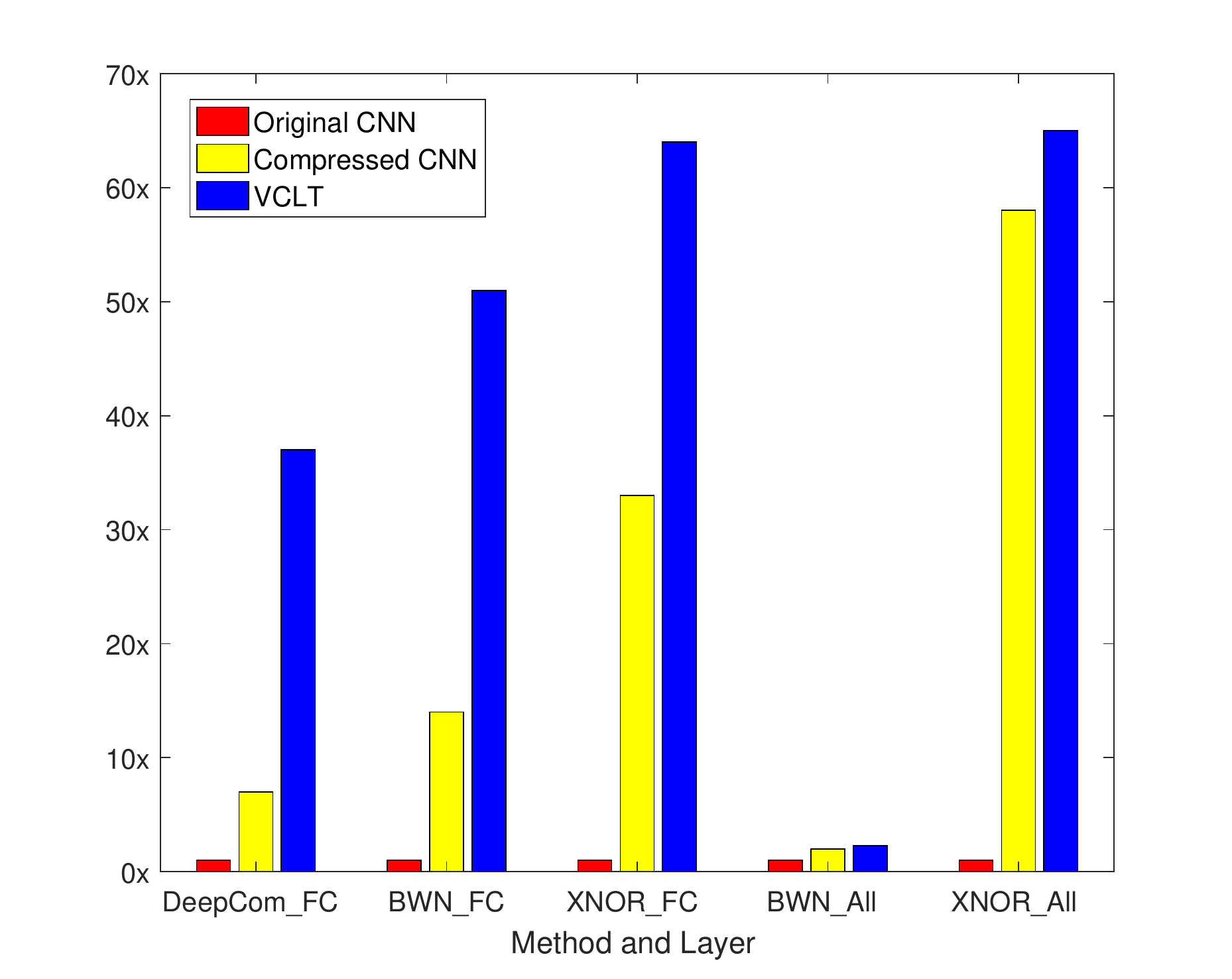}}
\end{minipage}
\caption{The VCLT for CIFAR-100 image set with 100 categories.}
\label{fig:speedupCompressed}
\end{figure}

Addtionally, we compare the classification accuracy and the speedup ratio of our VCLT with compressed CNN models. Here we choose the Binary-Weights-Net(BWN), XNOR-Net(XNOR)~\cite{Rastegari2016XNOR} and the DeepCompression Network(DeepCom)~\cite{Han2016Deep} for comparison. These compressed CNN models are based on AlexNet and we use their pre-trained models on \emph{ImageNet} in our experiment. The accuracy results are shown in Table \ref{tab:accuracyComparedCompressedCNN}, we find out that our VCLT has no accuracy decline.

The speedup ratio results are shown in Fig. 3. For DeepCom, followed~\cite{Han2016Deep}, comparisons of speedup mainly focus on FC layers and results are shown in \emph{DeepCom\_FC} in Fig. $3$. We find the speedup ratio of our VCLT$(37\times)$ is $30\times$ higher than DeepCom$(7\times)$ when compared with FC layers in the original AlexNet model. For BWN and XNOR, followed~\cite{Rastegari2016XNOR}, comparisons of speedup ratio focus on both FC layers and the entire network. The results are shown respectively in \emph{BWN\_FC}, \emph{XNOR\_FC}, \emph{BWN\_All} and \emph{XNOR\_All}. We find the speedup ratio of our VCLT is $37\times$ higher than BWN$(14\times)$ and $31\times$ higher than XNOR$(33\times)$ when compared on FC layers in the original AlexNet model. As for the entire network model, our VCLT$(2.3\times)$ obtain an improvement over the BWN$(2\times)$ by $0.3\times$ while VCLT$(65\times)$ over the XNOR$(58\times)$ by $7\times$ in terms of speedup ratio.

\section{Conclusion}

In this paper, we propose a method of replacing the fully-connected layers in CNN models with a tree classifier in image classification applications. We utilize the community detection algorithm to construct a Visual Confusion Label Tree based on the confusion characteristics of CNN models. Then, we use the multi-kernel SVM plus classifier with hierarchical constraints to train the tree classifier on the Visual Confusion Label Tree. Finally, we use this tree classifier to replace fully-connected layers in CNN models. The experimental results on \emph{CIFAR-100} and \emph{ImageNet} demonstrated the advantages of the proposed method over other tree classifiers and original CNN models such as AlexNet and VGG16.

\section{Acknowledgements}
This work was supported by the Natural Science Foundation of China under the grant No. U1435219, No. 61402507 and No. 61303070.

\bibliographystyle{IEEEbib}
\bibliography{camera-ready_icme2018template}

\clearpage
\begin{appendixpage}
	\begin{appendices}
		\setcounter{equation}{0}
		\setcounter{table}{0}
		\setcounter{figure}{0}
		\section{Theoretical Analysis of the Effect on Accuracy when Tree Structure Changed}
		
		We consider a problem of three categories classification. We assume that the categories are A, B and C and we can construct four different tree structures, which are shown in Fig.1. According to the definition of the Label Tree, tree structure $T_1$ is more reasonable than others for the classification task.

		\begin{figure}[!hb]

        	\begin{minipage}[b]{1.0\linewidth}
        	  \centering
        	  \centerline{\epsfig{figure=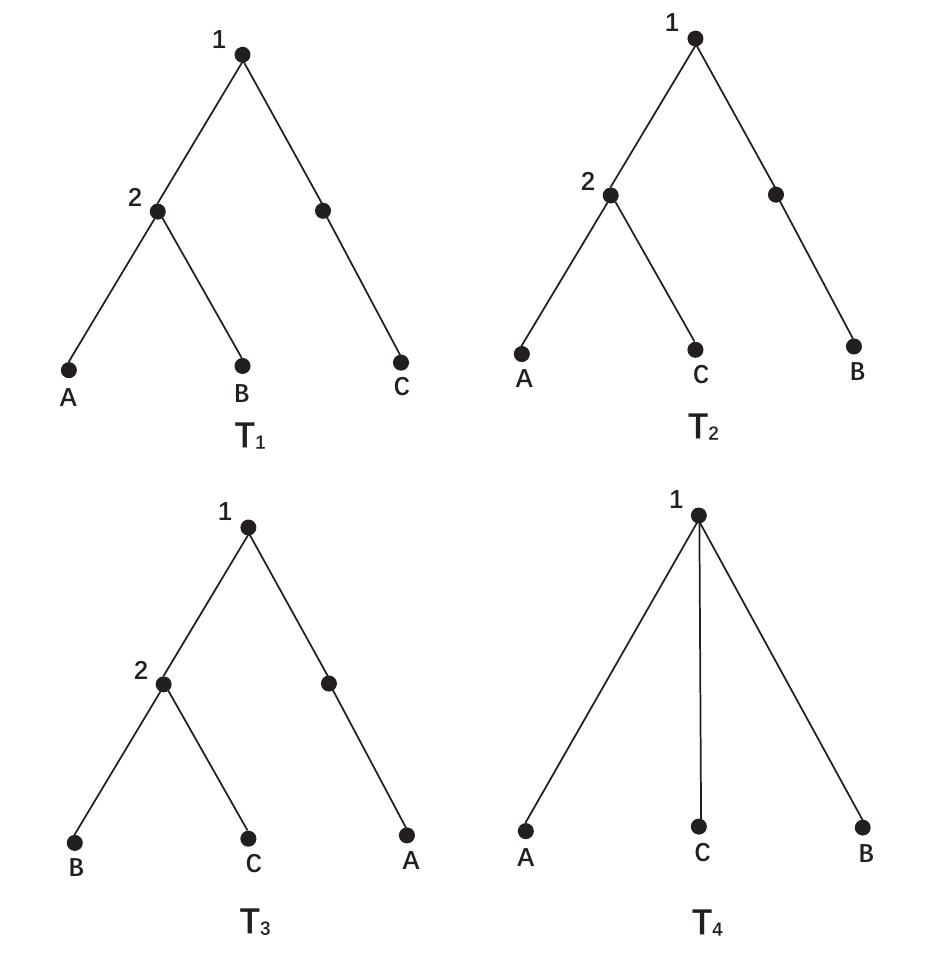,width=8cm}}
        	\end{minipage}
        	
        	\caption{Four tree structures.}
        	\label{fig:res}
        \end{figure}
				
		In order to prove it, we assume that we have trained classifiers on these tree structures. At each node of these tree structures, we use SVM for classifier and the SVM classifiers on these nodes are same. Here we assume the distances between every pair of three categories are $d_{AB}$, $d_{AC}$, $d_{BC}$ and $d_{BC}>d_{AC}\gg d_{AB}$, which means category A is similar to B while C is different from both of them. Ideally, the distances in high feature space which is projected by classifier(distance of SVM) among these categories is shown in Fig. 2.
		\begin{figure}[t]

        	\begin{minipage}[b]{1.0\linewidth}
        	  \centering
        	  \centerline{\epsfig{figure=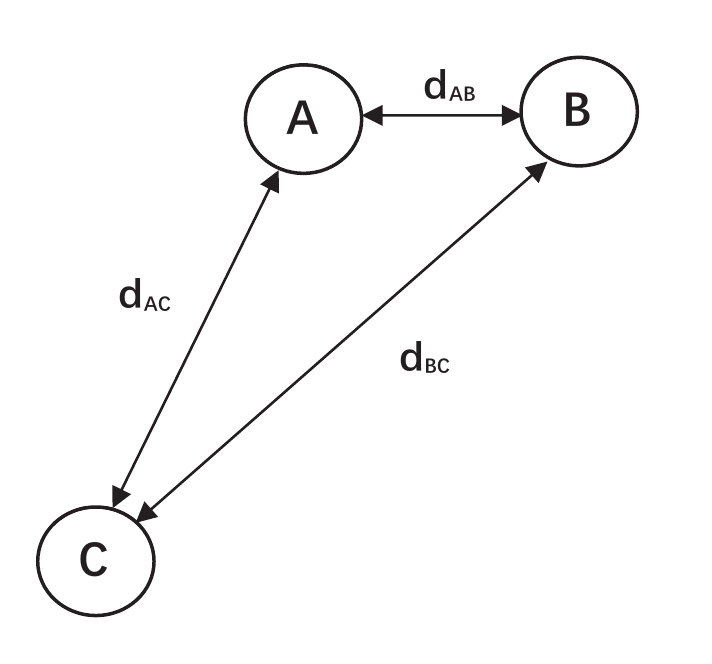,width=5cm}}
        	\end{minipage}
        	
        	\caption{The distances among category A, B and C.}
        	\label{fig:res}
        \end{figure}
		
		We will prove the tree structure $T_1$ is the most ideal tree structure. We choose one of $T_2$ and $T_3$ to be compared with $T_1$ because $T_2$ and $T_3$ are actually the same.

		\newtheorem{mypro}{Proposition}[section]
		
		\begin{mypro}
		Tree structure $T_1$ is better than $T_2$ and $T_4$.
		\end{mypro}

		\begin{proof}
		
		For $T_1$, we assume that $P^1_{AB}$ denotes the probability of classifying samples in A and B from all the samples correctly at the SVM on the node $1$. Then we define $P^1_C$, $P^2_A$, $P^2_B$ and so on. In addition, we define $P_A$, $P_B$ and $P_C$ as the probability of classifying the three categories from all the samples correctly. Here we know:
		
		\begin{equation}
			\begin{array}{l}
			P_A = P^1_{AB}\times P^2_A \\
			P_B = P^1_{AB}\times P^2_B \\
			P_C = P^1_C
			\end{array}
		\end{equation}
		
		Because the probability that SVM makes a correct classification is in proportion to the distance of SVM. So we get:
		
		\begin{equation}
			\begin{array}{l}
			P_A\propto d_{AC}\times d_{AB} \\
			P_B\propto d_{AC}\times d_{AB} \\
			P_C\propto d_{AC}
			\end{array}
		\end{equation}
		
		Similarly, for $T_2$, we get:
		
		\begin{equation}
			\begin{array}{l}
			P_A\propto d_{AB}\times d_{AC} \\
			P_B\propto d_{AB} \\
			P_C\propto d_{AB}\times d_{AC}
			\end{array}
		\end{equation}
		
		For $T_4$, we get:
		
		\begin{equation}
			\begin{array}{l}
			P_A\propto d_{AB} \\
			P_B\propto d_{AB} \\
			P_C\propto d_{AC}
			\end{array}
		\end{equation}
		
		And we know the probability of correct classification for a tree structure $T$ can be defined as:
		
		\begin{equation}
			P_{T}=P_A+P_B+P_C
		\end{equation}

		The probabilities for $T_1$, $T_2$ and $T_4$ are:
		
		\begin{equation}
        \label{equ:pWithoutk}
			\begin{array}{l}
			P_{T_1}\propto (2d_{AC}\times d_{AB}+d_{AC}) \\
			P_{T_2}\propto (2d_{AC}\times d_{AB}+d_{AB}) \\
			P_{T_3}\propto (2d_{AB}+d_{AC})
			\end{array}
		\end{equation}

		The SVM classifiers on the nodes of a tree structure are same so the probabilities in Eq.(\ref{equ:pWithoutk}) have a same proportion, we denote it as $k$:
		
		\begin{equation}
			\begin{array}{l}
			P_{T_1}=k(2d_{AC}\times d_{AB}+d_{AC}) \\
			P_{T_2}=k(2d_{AC}\times d_{AB}+d_{AB}) \\
			P_{T_3}=k(2d_{AB}+d_{AC})
			\end{array}
		\end{equation}
		
		For $P_{T_1}$ and $P_{T_2}$:
		
		\begin{equation}
			\begin{array}{l}
			P_{T_1}-P_{T_2}=k(d_{AC}-d_{AB}) \\
			\because d_{AC}>d_{AB} \\
			\therefore P_{T_1} > P_{T_2}
			\end{array}
		\end{equation}
		
		For $P_{T_1}$ and $P_{T_4}$:
		
		\begin{equation}
			P_{T_1}-P_{T_4}=k(2d_{AB}(d_{AC}-1))
		\end{equation}

		If $d_{AC}>1$ then $P_{T_1}>P_{T_4}$, which demands $d_{AC}$ is very large. From the assumption we know that category C is far from category B and C in the distance of SVM, so $P_{T_1}>P_{T_4}$.

		In summary, Tree structure $T_1$ is better than $T_2$, $T_3$ and $T_4$.
		
		\end{proof}

		\section{Theoretical Analysis of the speedup ratio on replacing Fully-Connected layers with the Tree Classifier}
		In this section we talk about the speedup ratio of our method in theory. Here we take AlexNet on \emph{CIFAR-100} and \emph{ImageNet} datasets as the analysis object and we replace fully-connected(FC) 7 and FC8 layers with the tree classifier. As we all know, each layer of the FC layers is actually a vector inner product process, which can be defined as:
		\begin{equation}
		Out_m=\sum^N_{i=1}{C(w_m,f_i)}+b_m,m=1,2,...,M
		\end{equation}
		with:
		\begin{equation}
		C(w,f)=\sum^n_{i=1}{\sum^n_{j=1}{w_{ij}f_{ij}}}
		\end{equation}
		where $Out$ is the output of the FC layer, $w$ is the convolution kernel, $f$ is the input of the FC layer, $b$ is the bias, their subscript $m$ denotes that it is the output of the $m$th channel and there are $M$ channels in total. $C$ means convolution operation and $C(w,f)$ denotes a convolution operation between kernel $w$ with a size of $n\times n$ and a feature map $f$. As for FC7 and FC8, the inputs both are feature maps with a size of $1\times 1\times 4096$ and the outputs are a feature map with the size of $1\times 1\times 4096$ and a score vector with the size of $1\times 1\times 1000$. We can calculate the computation using the equation mentioned above.
		
		If we replace the FC layers with our tree classifier, then we should calculate the computation of the tree classifier. The tree classifiers have a hierarchical structure. For one node there is a classifier on each of the child nodes under it. Each classifier under the same parent node computes a result of one test image and the parent node selects which branch to go by comparing all these results. We repeat this process from root node to leaf nodes. Finally the category on the selected leaf node is the classification result. Therefore, the number($N$) of classifiers involved in the classification process is equal to the sum of the number of child nodes under all the nodes in the path from the root node to a specific leaf node. Here we find that the worst situation is the path with the most child nodes and we find this specific number of CIFAR-100 and ImageNet is $18$ and $65$. The dimension($d$) of the features which is used for the classifier is $4096$. Therefore, we can calculate the multi-adds computation of the tree classifier on \emph{CIFAR-100} dataset is:

        \begin{equation}
            2Nd=2\times 18\times 4096 = 147456
        \end{equation}
Similarly, we can calculate the computation on \emph{ImageNet} dataset is $524288$.
		
		We make a statistic comparison in Table \ref{tab:theoryComputaion}.
		
		\begin{table}[t]
		    \centering
		    \begin{tabular}{|l|l|l|}
		        \hline
		        Classifier & CIFAR-100 & ImageNet \\
		        \hline
		        FC & 34 & 42 \\
		        Tree Classifier & 0.14 & 0.52 \\
		        Speedup & 233$\times$ & 78$\times$\\
		        \hline
		    \end{tabular}
		    \caption{Comparison of the computation on different datasets in theory(unit:million mult-adds)}
		    \label{tab:theoryComputaion}
		\end{table}

	\end{appendices}
\end{appendixpage}

\end{document}